%% file: main.tex
\pdfoutput=1
\documentclass[letterpaper, 10 pt, conference]{cls/ieeeconf}

\IEEEoverridecommandlockouts  
\let\labelindent\relax

\usepackage[utf8]{inputenc}
\usepackage[T1]{fontenc}
\usepackage[noadjust]{cite}
\usepackage{amsmath,amssymb,amsfonts,amsthm,amscd,dsfont}
\usepackage{mathtools}
\usepackage{accents} 
\usepackage{algorithm,listings}
\usepackage[noend]{algpseudocode}
\usepackage{xcolor}
\usepackage{graphicx,tabularx,adjustbox}
\usepackage[font={small}]{caption}
\usepackage{subcaption}
\usepackage{paralist}

\usepackage[titletoc,title]{appendix}
\usepackage{balance}
\usepackage[breaklinks=true, colorlinks, bookmarks=true]{hyperref}
\usepackage{pbox}

\newcommand{\crl}[1]{\left\{#1\right\}}

\newtheorem{proposition}{Proposition}

\theoremstyle{definition}
\newtheorem{definition}{Definition}

\newtheorem*{problem}{Problem}
\theoremstyle{remark}

\input{cls/sym.tex}

\title{\LARGE \bf Optimal Scene Graph Planning with Large Language Model Guidance}

\author{\authorblockN{Zhirui Dai$^{1}$ \qquad Arash Asgharivaskasi$^{1}$ \qquad Thai Duong$^{1}$ \qquad Shusen Lin$^{1}$ \qquad Maria-Elizabeth Tzes$^{2}$} \\[-2ex] \authorblockN{George Pappas$^{2}$ \qquad Nikolay Atanasov$^{1}$}
\thanks{We acknowledge support from ARL DCIST CRA W911NF-17-2-0181.}%
\thanks{$^{1}$The authors are with the Department of Electrical and Computer Engineering, University of California San Diego, La Jolla, CA 92093, USA, e-mails: {\tt\small \{zhdai,\allowbreak shl090,\allowbreak tduong,\allowbreak aasghari,\allowbreak natanasov\}@ucsd.edu}.}%
\thanks{$^{2}$The authors are with GRASP Lab, University of Pennsylvania, Philadelphia, PA 19104, USA, e-mails: {\tt\small \{mtzes,\allowbreak pappasg\}\allowbreak@seas.upenn.edu}.}%
}

\begin{document}
\maketitle
\thispagestyle{empty}  
\pagestyle{empty}      

\begin{abstract}
Recent advances in metric, semantic, and topological mapping have equipped autonomous robots with semantic concept grounding capabilities to interpret natural language tasks. This work aims to leverage these new capabilities with an efficient task planning algorithm for hierarchical metric-semantic models. We consider a scene graph representation of the environment and utilize a large language model (LLM) to convert a natural language task into a linear temporal logic (LTL) automaton. Our main contribution is to enable optimal hierarchical LTL planning with LLM guidance over scene graphs. To achieve efficiency, we construct a hierarchical planning domain that captures the attributes and connectivity of the scene graph and the task automaton, and provide semantic guidance via an LLM heuristic function. To guarantee optimality, we design an LTL heuristic function that is provably consistent and supplements the potentially inadmissible LLM guidance in multi-heuristic planning. We demonstrate efficient planning of complex natural language tasks in scene graphs of virtualized real environments.
\end{abstract}


\input{tex/Introduction.tex}

\input{tex/ProblemStatement.tex}

\input{tex/NL_to_Automaton}
\input{tex/Planning}
\input{tex/Evaluation.tex}
\input{tex/Conclusion.tex}






\balance
{\small
\bibliographystyle{cls/IEEEtran}
\bibliography{bib/main.bib}
}

\end{document}

%% file: cls/sym.tex
\newcommand{\bfF}{\mathbf{F}}
\newcommand{\bfG}{\mathbf{G}}

\newcommand{\bfU}{\mathbf{U}}

\newcommand{\bfX}{\mathbf{X}}

\newcommand{\bbN}{\mathbb{N}}

\newcommand{\bbR}{\mathbb{R}}

\newcommand{\calA}{\mathcal{A}}

\newcommand{\calE}{\mathcal{E}}
\newcommand{\calF}{\mathcal{F}}
\newcommand{\calG}{\mathcal{G}}

\newcommand{\calM}{\mathcal{M}}

\newcommand{\calP}{\mathcal{P}}
\newcommand{\calQ}{\mathcal{Q}}
\newcommand{\calR}{\mathcal{R}}

\newcommand{\calT}{\mathcal{T}}
\newcommand{\calU}{\mathcal{U}}
\newcommand{\calV}{\mathcal{V}}

\newcommand{\calX}{\mathcal{X}}

\newcommand{\AP}{\calA\!\calP}



%% file: tex/Introduction.tex
\section{INTRODUCTION}
\label{sec:introduction}

Advances in robot perception and computer vision have enabled metric-semantic mapping \cite{semantic_fusion,Bowman_SemanticSLAM_ICRA17,dong2017semantic,Shan_OrcVIO_IROS20,kimera,Zobeidi_GPMapping_TRO22,Asgharivaskasi_ActiveMulticlassMapping_TRO23,NeuralBKI,scenecode}, offering rich information in support of robot autonomy. Beyond single-level maps, hierarchical models encode topological relations among local maps and semantic elements \cite{hierarchical_slam,armeni20193d}. A scene graph \cite{armeni20193d} is a prominent example that models buildings, floors, rooms, objects, and occupancy in a unified hierarchical representation. Scene graph construction can be done from streaming sensor data \cite{rosinol20203d,kimera,hughes2022hydra}. The metric, semantic, and topological elements of such models offer the building blocks for robots to execute semantic tasks \cite{crespo2020semanticsurvey}. 
The objective of this work is to approach this challenge by generalizing goal-directed motion planning in flat geometric maps to natural language task planning in scene graphs.

Connecting the concepts in a natural language task to the real-world objects they refer to is a challenging problem, known as symbol grounding \cite{tellex2011approaching}. Large language models (LLMs), such as GPT-3 \cite{brown2020language}, BERT \cite{devlin2018bert}, and LLaMA \cite{touvron2023llama}, offer a possible resolution with their ability to relate entities in an environment model to concepts in natural language. Chen et al. \cite{chen2022extracting,chen2022leveraging} use LLMs for scene graph labeling, showing their capability of high-level understanding of indoor scenes. Shah et al. \cite{shah2022lm} use GPT-3 to parse text instructions to landmarks and the contrastive language image pre-training (CLIP) model \cite{clip} to infer a joint landmark-image distribution for visual navigation. Seminal papers \cite{symbolic_planning,ltl_planning} in the early 2000s established formal logics and automata as powerful representations of robot tasks. In work closely related to ours, Chen et al. \cite{chen2023autotamp} show that natural language tasks in 2D maps encoded as sets of landmarks can be converted to signal temporal logic (STL) \cite{stl} via LLM re-prompting and automatic syntax correction, enabling the use of existing temporal logic planners \cite{tulip,vasile2013samplingbased,fu2016optimal,purohit2021dt,stylus}. Beyond temporal logics, other expressive robot task representations include the planning domain definition language \cite{PDDL}, Petri nets \cite{jongwook1995,LV20229}, and process algebra \cite{uav_process_algebra}. However, few of the existing works have considered task specification in hierarchical models. We use an LLM to translate a natural language task defined by the attributes of a scene graph to a linear temporal logic (LTL) formula. We describe the scene graph to the LLM via an attribute hierarchy and perform co-safe syntax checking to ensure generation of correct and finite-path verifiable LTL. Further, we develop an approach to obtain task execution guidance from the LLM, which is used to accelerate the downstream task planning algorithm.

\begin{figure*}[t]
    \centering
    \includegraphics[width=0.94\linewidth, trim=0mm 2mm 0mm 0mm,clip]{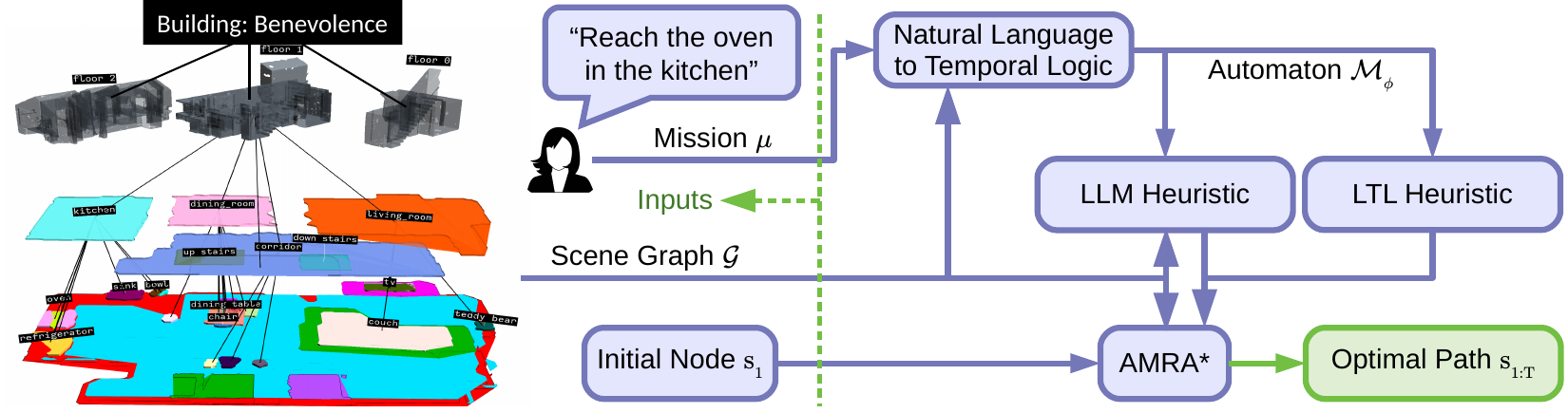}
    \caption{Planning a natural language mission, $\mu: \text{``Reach the oven in the kitchen''}$, in a scene graph $\calG$ of the Gibson environment Benevolence \cite{xiazamirhe2018gibsonenv} with object, room, and floor attributes. The terms ``oven'' and ``kitchen'' in $\mu$ belong to the object and room attributes of the scene graph, respectively. The scene graph $\calG$ is described to the LLM using the connectivity of its attributes (attribute hierarchy) and the LLM is used to translate $\mu$ to LTL formula $\phi_{\mu}$ and associated Automaton $\calM_{\phi}$. We construct a hierarchical planning domain from the scene graph, and use multi-resolution multi-heuristic planning \cite{saxena2022amra} to plan the mission execution. In addition to mission translation, the LLM is used to provide heuristic guidance to accelerate the planning, while an LTL heuristic is used to guarantees optimality.}
\label{fig:example_mission}
\end{figure*}

Given a symbolically grounded task representation, the next challenge is to plan its execution in a hierarchical model efficiently. A key component for achieving efficiency in traditional goal-oriented planning in single-level environments is the use of guidance from a heuristic function. Heuristics play a critical role in accelerating both search-based \cite{astar,aine2016multi} and sampling-based \cite{informed_rrt,bit_star} planners. More complex tasks than goal reaching involve sequencing, branching, and recurrence, making heuristic guidance even more important for efficiency. Our work is inspired by Fu et al.~\cite{fu2016optimal} who develop an admissible heuristic for LTL planning in probabilistic landmark maps. We extend the approach by (a) generalizing to hierarchical scene graphs via multi-resolution planning, (b) designing a consistent LTL heuristic allowing acceleration over admissible-only heuristic planning, and (c) introducing an LLM heuristic allowing acceleration from LLM semantic guidance while retaining \emph{optimality guarantees} via multi-heuristic planning. Our approach is enabled by the anytime multi-resolution multi-heuristic A* (AMRA*) \cite{saxena2022amra}, which combines the advantages of multi-resolution A* \cite{du2020multi} and multi-heuristic A* \cite{aine2016multi}. Our key contribution is to define the nodes, edges, and costs of a hierarchical planning domain from a scene graph and to introduce guidance from a consistent LTL heuristic and a semantically informed LLM heuristic.

Related works consider object search and semantic goal navigation in unknown environments, represented as semantic occupancy maps \cite{shah2023lm}, topological maps \cite{kostavelis2016robot}, or scene graphs \cite{amiri2022reasoning,ravichandran2022hierarchical}. Shah et al. \cite{shah2023lm} develop an exploration approach that uses an LLM to score subgoal candidates and provide an exploration heuristic. Kostavelis et al. \cite{kostavelis2016robot} perform place recognition using spatial and temporal proximity to obtain a topological map and encode the connectivity of its attributes in a navigation graph to enable Dijkstra planning to semantically specified goals. Amiri et al. \cite{amiri2022reasoning} employ a scene graph generation network to construct a scene graph from images and a partially observable Markov decision process to obtain an object-search policy. Ravichandran et al. \cite{ravichandran2022hierarchical} embed partial scene graph observations in feature space using a graph neural network (GNN) and train value and policy GNNs via proximal policy optimization. In contrast with these works, we consider significantly more general missions but perform planning in a known scene graph. 

In summary, this paper makes the following \emph{contributions}.
\begin{itemize}
    
    \item We use an LLM to translate natural language to LTL tasks over the attributes of a scene graph.
    
    \item We construct a hierarchical planning domain capturing the structure of the scene graph and LTL task.
    
    \item We design new LTL and LLM heuristic functions for planning, and prove that the LTL heuristic is consistent.

    \item We employ hierarchical multi-heuristic planning to guarantee efficiency (due to LLM semantic guidance) and optimality (due to LTL consistent guidance), despite potential inadmissibility of the LLM heuristic.
\end{itemize}

%% file: tex/ProblemStatement.tex
\section{PROBLEM STATEMENT}
\label{sec:problem_statement}

Consider an agent planning a navigation mission specified in terms of semantic concepts, such as objects and regions, in a known environment. We assume that the environment is represented as a scene graph \cite{armeni20193d}.

\begin{definition}
A \emph{scene graph} is a graph $\calG = (\calV,\calE, \{\calA_k\}_{k=1}^K)$ with node set $\calV$, edge set $\calE \subseteq \calV \times \calV$, and $K$ attribute sets $\calA_k$ for $k = 1,\ldots,K$. Each attribute $a \in \calA_k$ is associated with a subset $\calV_a$ of the nodes $\calV$.
\end{definition}

A scene graph provides a hierarchical semantic description of an environment, such as a building, in terms of floors, rooms, objects, and occupancy (see Fig.~\ref{fig:example_mission}). For example, the graph nodes $\calV$ may represent free space, while the objects may be encoded as an attribute set $\calA_{1}$ such that an object $o \in \calA_1$ is associated with a free region $\calV_o \subset \calV$ around it. Similarly, rooms may be represented as a set $\calA_2$ such that a room $r \in \calA_2$ is associated with a free-space region $\calV_r \subset \calV$. 

Given the initial agent location $s \in \calV$, and a cost function $c : \calE \mapsto \bbR_{>0}$, our objective is to plan a sequence of scene graph nodes (path) that achieves a mission $\mu$ with minimal cost. We assume $\mu$ is provided in natural language using terms from the attribute sets $\calA_k$ of the scene graph. An example scene graph mission is provided in Fig.~\ref{fig:example_mission}.

To interpret a natural language mission, we define atomic propositions whose combinations can capture the mission requirements. An \emph{atomic proposition} $p_a: \calV \mapsto \{\mathsf{false},\mathsf{true}\}$ associated with attribute $a \in \calA_k$ of the scene graph $\calG$ evaluates true at node $s \in \calV$ if $s$ belongs to the nodes $\calV_a$ that satisfy attribute $a$. We denote this by $ p_a(s) : s \in \calV_a$.
The set of all atomic propositions at $s \in \calV$ is denoted by:
\begin{equation}\label{eq:atomic_propositions}
    \small
    \AP(s) \coloneqq \crl{ p_a(s) \mid a \in \calA_k, k = 1,\ldots,K}.
\end{equation}
Intuitively, $p_a(s)$ being $\mathsf{true}$ means that the agent is at node $s$ that satisfies attribute $a$, e.g., reaching an object in $\calA_1$ or being in a room in $\calA_2$. Avoiding an object or leaving a room can be specified via the negation of such propositions. To determine which atomic propositions are satisfied at a particular node, we define a labeling function. 

\begin{definition}
Consider a scene graph $\calG$ with atomic propositions $\AP = \cup_{s \in \calV} \AP(s)$. A \emph{labeling function} $\ell : \calV \mapsto 2^{\AP}$ maps a node $s \in \calV$ to a set $\ell(s) \subseteq \AP$ of atomic propositions that evaluate $\mathsf{true}$ at $s$.
\end{definition}

The labels along a path $s_{1:T}$ are obtained as $\ell(s_{1:T}) = \ell(s_1)\ell(s_2)\ldots\ell(s_T)$ and are called a \emph{word}. A word contains information about the objects, rooms, and floors that an agent visits along its path in a scene graph. We can decide whether the agent path satisfies a mission $\mu$ by interpreting its word. We denote a word $\ell(s_{1:T})$ that satisfies a mission $\mu$ by $\ell(s_{1:T}) \models \mu$, and define the precise semantics of this notation in Sec.~\ref{sec:nl_to_ltl}. With this, we are ready to state the problem of natural language mission planning in scene graphs.

\begin{problem}
Given a scene graph $\calG = (\calV,\calE, \{\calA_k\}_{k=1}^K)$, natural language mission $\mu$ over the attributes of $\calG$, cost function $c : \calE \mapsto \bbR_{>0}$, and initial node $s_1 \in \calV$, plan a path $s_{1:T}$ that satisfies $\mu$ with minimal cost:
\begin{equation}
\begin{aligned}
    \min_{T \in \bbN, s_{1:T} \in \calV^T } & \sum_{t=1}^{T-1} c(s_t,s_{t+1})\\
    \text{s.t.} \quad & (s_t,s_{t+1}) \in \calE, \quad t = 1,\ldots,T-1,\\
    &\ell(s_{1:T}) \models \mu.
\end{aligned} \label{eq:problem_statement}
\end{equation}
\end{problem}

%% file: tex/NL_to_Automaton.tex
\section{NATURAL LANGUAGE TO TEMPORAL LOGIC}
\label{sec:nl_to_ltl}

We use an LLM to translate natural language missions to logic formulas over the scene graph propositions $\AP$. The key challenge is to describe the structure of a scene graph $\calG$ to an LLM and ask it to translate a mission $\mu$ to a logic formula $\phi_{\mu}$. We focus on linear temporal logic (LTL) \cite{pnueli1981temporal} with syntax in Table.~\ref{tbl:ltl_gram} due to its popularity and sufficient expressiveness to capture temporal ordering of propositions.
 
We require a syntactically co-safe formula \cite{kupferman2001model} to allow checking its satisfaction over finite agent paths. A co-safe LTL formula can be satisfied by a word $\ell(s_{1:T})$ that consists of a finite prefix followed by a (potentially infinite) continuation that does not affect the formula's truth value. LTL formulas that contain only $\bfX$ and $\bfU$ temporal operators when written in negated normal form ($\neg$ appears only in front of atomic propositions) are syntactically co-safe \cite{kupferman2001model}.

To use an LLM for scene understanding, it is necessary to design a prompt that describes the scene graph $\calG$ succinctly. Otherwise, the input might exceed the model token limit or confuse the model about the relationship between the sentences. For this aim, we simplify the scene graph $\calG$ into an \emph{attribute hierarchy} $\Bar{\calG}$ that compactly represents scene entities in a YAML format. In our setup, the top level of $\Bar{\calG}$ contains floors $f \in \calA_3$. The rooms $r_f$ on floor $f$ are defined as $\crl{r \in \calA_2 | \calV_r \subseteq \calV_f}$, and nested as children of floor $f$. Additionally, each room $r$ stores connections to other rooms on the same floor. Similarly, the objects in room $r$, $\crl{o \in \calA_1 | \calV_o \subseteq \calV_r}$, are stored as children of room $r$. Each entity in $\Bar{\calG}$ is tagged with a unique ID to differentiate rooms and objects with the same name. See Fig.~\ref{fig:nl_to_automaton}a for an example attribute hierarchy. In our examples, we define attributes for floors, rooms, and objects and the attribute hierarchy contains $3$ levels but this can be extended to more generalized attribute sets (e.g., level for object affordances). The attribute hierarchy removes the dense node and edge descriptions in $\calG$ which are redundant for mission translation, leading to a significant reduction in prompt size.

Given the natural language mission $\mu$ and the attribute hierarchy $\Bar{\calG}$, the first call to the LLM is only responsible to extract unique IDs from the context of $\mu$, outputting $\mu_{\text{unique}}$. This step facilitates LTL translations by separating high-level scene understanding from accurate LTL generation. Specifically, this step links contextually rich specifications, which can be potentially difficult to be parsed, to unequivocal mission descriptions that are void of confusion for the LLM when it comes to LTL translation (Fig.~\ref{fig:nl_to_automaton}b). The list of entities involved in the mission are extracted from $\mu_{\text{unique}}$ using regular expression, and stored as $\mu_{\text{regex}}$. This allows to inform the LLM about the essential parts of the mission $\mu_{\text{unique}}$, without providing $\Bar{\calG}$. The savings in prompt size are used to augment the prompt with natural language to LTL translation examples expressed in \emph{pre-fix} notation. For instance, $\phi \wedge \varphi$ is expressed as $\wedge \phi \varphi$ in pre-fix format. This circumvents the issue of balancing parenthesis over formula $\phi_{\mu}$. Ultimately, the LTL translation prompt includes $\mu_{\text{unique}}$, $\mu_{\text{regex}}$, and the translation examples (Fig.~\ref{fig:nl_to_automaton}c).

\begin{table}[t]
\caption{\footnotesize Grammar for LTL formulas $\phi$ and $\varphi$.}
\scriptsize
\label{tbl:ltl_gram}
\centering
\begin{tabular}{cccccc}
\hline
\multicolumn{6}{c}{Syntax (Semantics)}\\
\hline
$p_a$ & (Atomic Proposition) & $\phi \vee \varphi$ & (Or) & $\phi \bfU \varphi$ & (Until) \\
$\neg \phi$ & (Negation) & $\phi \Rightarrow \varphi$ & (Imply) & $\bfF \phi$ & (Eventually) \\
$\phi \wedge \varphi$ & (And) & $\bfX \phi$ & (Next) & $\bfG \phi$ & (Always) \\
\hline
\end{tabular}
\end{table}

The translated LTL formula $\phi_{\mu}$ is verified for syntactic correctness and co-safety using an LTL syntax checker. Further calls of the LLM are made to correct $\phi_{\mu}$ if it does not pass the checks, up to a maximum number of allowed verification steps, after which human feedback is used to rephrase the natural language specification $\mu$ (Fig.~\ref{fig:nl_to_automaton}d). Once the mission $\mu$ is successfully translated into a co-safe LTL formula $\phi_{\mu}$, we can evaluate whether an agent path $s_{1:T}$ and its corresponding word $\ell(s_{1:T})$ satisfy $\phi_{\mu}$ by constructing an automaton representation of $\phi_{\mu}$ (Fig.~\ref{fig:nl_to_automaton}e).

\begin{definition}
A deterministic automaton over atomic propositions $\AP$ is a tuple $\calM = (\calQ, 2^{\AP}, T, \calF, q_1)$, where $\calQ$ is a set of states, $2^{\AP}$ is the power set of $\AP$ called alphabet, $T: \calQ \times 2^{\AP} \mapsto \calQ$ is a transition function that specifies the next state $T(q,l)$ from state $q \in \calQ$ under label $l \in 2^{\AP}$, $\calF \subseteq \calQ$ is a set of final states, and $q_1 \in \calQ$ is an initial state.
\end{definition}

A co-safe LTL formula $\phi$ can be translated into a deterministic automaton $\calM_{\phi_{\mu}}$ using model checking tools such as PRISM \cite{prism,klein2018advances} or Spot \cite{spot}. The automaton checks whether a word satisfies $\phi_{\mu}$. A word $\ell_{1:T}$ is accepted by $\calM_{\phi_{\mu}}$, i.e., $\ell_{1:T} \models \phi_{\mu}$, if and only if the state $q_{T+1}$ obtained after transitions $q_{t+1} = T(q_t,\ell_t)$ for $t = 1,\ldots,T$ is contained in the final states $\calF$. Hence, a path $s_{1:T}$ satisfies a co-safe LTL formula $\phi_{\mu}$ if and only if its word $\ell(s_{1:T})$ contains a prefix $\ell(s_{1:t})$ that is accepted by $\calM_{\phi_{\mu}}$. 

\begin{figure}[t]
    \centering
    \includegraphics[width=0.95\linewidth,trim=0mm 2mm 0mm 2mm,clip]{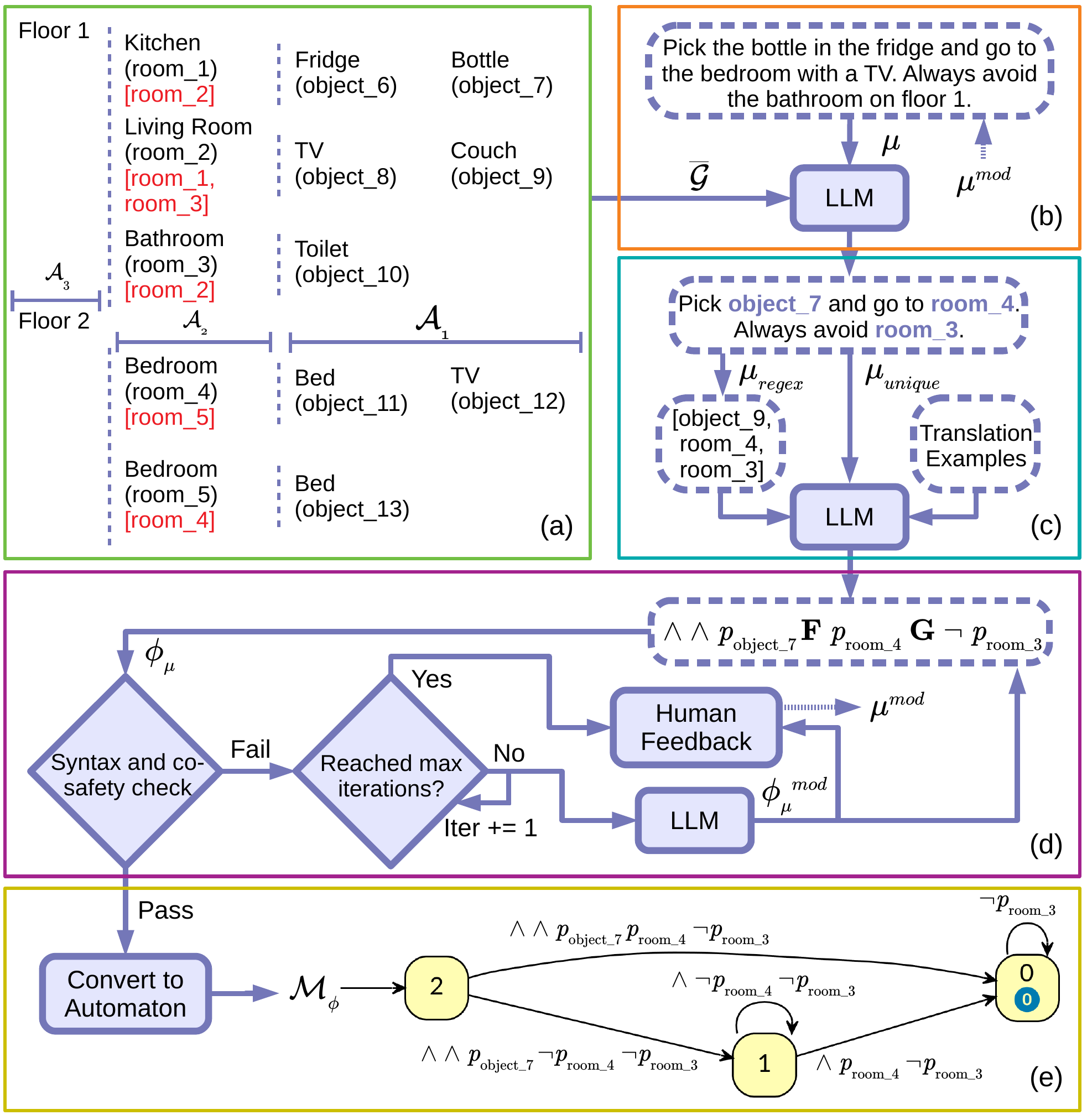}
    \caption{Natural language to LTL translation. (a) Attribute hierarchy $\Bar{\calG}$. The unique IDs and the room connections are shown in parenthesis and inside red brackets, respectively. (b) Unique ID extraction from natural language mission $\mu$. (c) LTL formula generation from natural language specification. (d) Syntax and co-safety check over the generated LTL formula $\phi_{\mu}$. (e) Automaton construction.}
    \label{fig:nl_to_automaton}
\end{figure}

%% file: tex/Planning.tex
\section{OPTIMAL SCENE GRAPH PLANNING}
\label{sec:planning}

We use the structure of the scene graph $\calG$ with guidance from the automaton $\calM_{\phi_{\mu}}$ and the LLM's mission semantics understanding to achieve efficient and optimal planning.

\subsection{AMRA* Planning} \label{sec:amra_planning}

We perform mission planning using AMRA* \cite{saxena2022amra}. The key challenge is to define a hierarchical planning domain and heuristic functions that describe the scene graph and mission. AMRA* requires several node sets $\calX_r$ and action sets $\calU_r$ for different planning resolution levels, 
$r = 1,2,\ldots$. Each level $(\calX_r,\calU_r)$ has an associated cost function $c_r:\calX_r\times \calX_r \mapsto \bbR_{>0}$. The algorithm defines an anchor level $(\calX_0,\calU_0)$, as $\calX_0 := \cup_{r >0} \calX_r$, $\calU_0 := \cup_{r > 0}\calU_r$, and requires an initial state $x \in \calX_0$ and a goal region $\calR \subseteq \calX_0$. AMRA* allows multiple heuristics for each level but requires the anchor-level heuristics to be consistent to guarantee optimality. A heuristic $h$ is consistent with respect to cost $c$ if $h(x) \leq c(x,x') + h(x')$. We construct the levels $\{\calX_r,\calU_r, c_r\}$, initial state $x$, and goal region $\calR$ required for running AMRA* from the scene graph $\calG$ and the automaton $\calM_{\phi_{\mu}}$ in Sec.~\ref{sec:planning_domain}. We define two heuristics, $h_{\textsc{LTL}}$, which is used in the anchor level and other levels, and $h_{\textsc{LLM}}$ for other levels, in Sec.~\ref{sec:ltl_heuristics} and Sec.~\ref{sec:llm_heuristics}, respectively, and prove that $h_{\textsc{LTL}}$ is consistent.

\subsection{Hierarchical Planning Domain Description} \label{sec:planning_domain}

Given scene graph $\calG$ with initial node $s_1 \in \calV$ and automaton $\calM_{\phi_{\mu}}$, we construct a hierarchical planning domain with $K$ levels corresponding to each scene graph attribute $\calA_k$. Given an attribute set $\calA_k$, we define $\calV_k := \cup_{a \in \calA_k} \calV_a$, where $\calV_a$ is the node set associated with attribute $a \in \calA_k$. Then, the node set corresponding to level $k$ of the planning domain is defined as $\calX_k := \calV_k \times \calQ$. We define the actions $\calU_k$ as transitions from $x_i$ to $x_j$ in $\calX_k$ with associated cost $c_k(x_i,x_j)$. A transition from $x_i=(s_i, q_i)$ and $x_j=(s_j, q_j)$ exists if the following conditions are satisfied:
\begin{enumerate}
    \item the transition is from $\calV_a$ of attribute $a$ to the boundary $\partial \calV_b$ of attribute $b$ such that $s_i \in \calV_a$, $s_j \in \partial \calV_b$ for $a \neq b$ with $a,b \in \calA_k$ and $\text{int}\calV_a \cap \text{int}\calV_b = \emptyset$,
    
    \item the automaton transitions are respected: $q_j = T(q_i, \ell(s_i))$, where $\ell(s_i) \in 2^{\AP}$ is the label at $s_i$,
    
    \item the transition is along the shortest path, $s_j \in \arg\min_{s} d(s_i, s)$, where $d: \calV \times \calV \rightarrow \bbR$ is the shortest feasible path between two scene graph nodes.
\end{enumerate}
The transition cost is defined as $c_k(x_i,x_j) = d(s_i, s_j)$.

\begin{figure}
    \centering
    \includegraphics[width=0.9\linewidth,trim=0mm 2mm 0mm 2mm,clip]{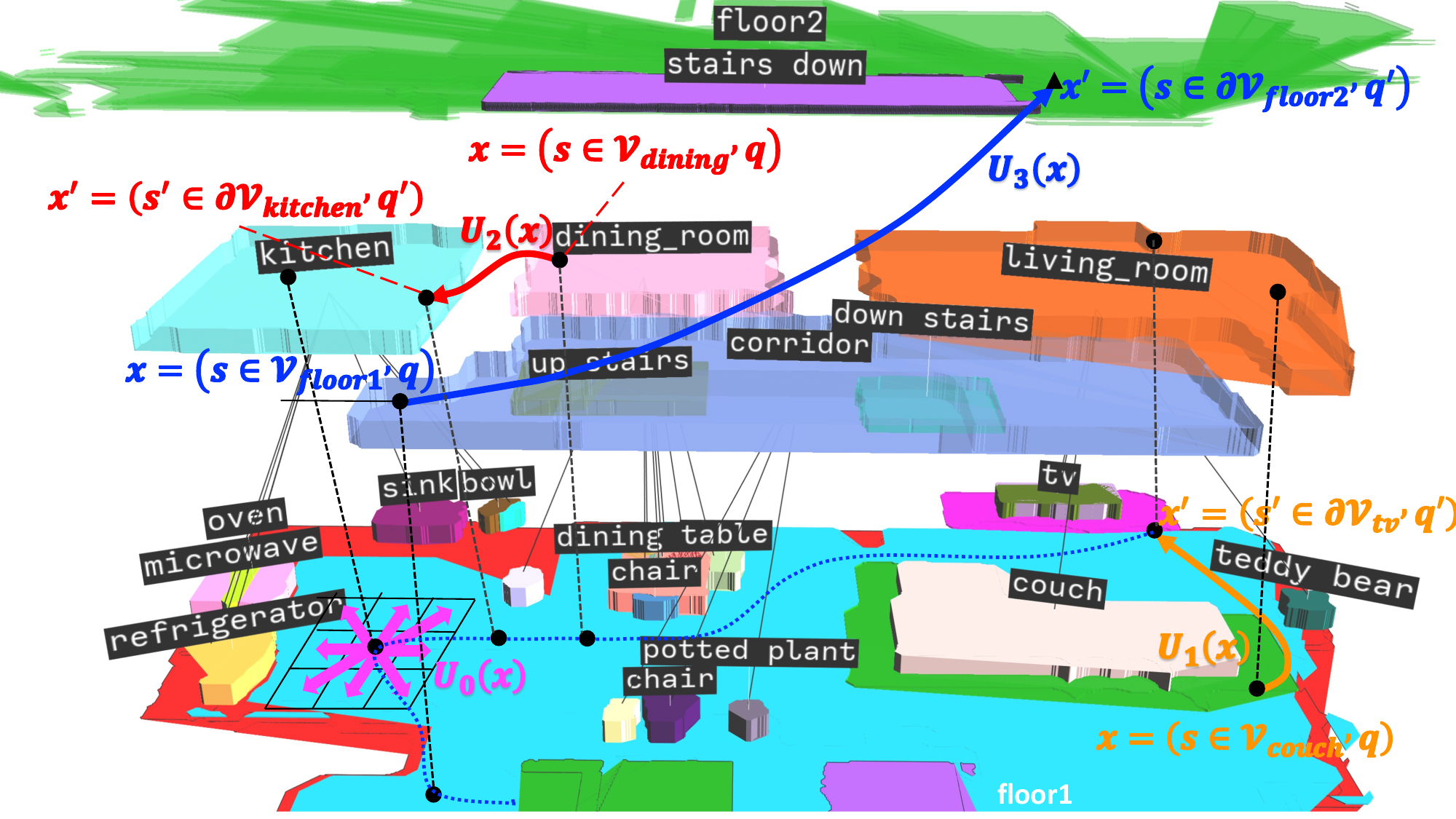}
    \caption{Four-level hierarchical planning domain for \emph{Benevolence}.}
    \label{fig:amra_structure}
\end{figure}

Since $\calX_k \subseteq \calV \times \calQ$, we can define the AMRA* anchor level as $\calX_0 = \calV \times \calQ$ with actions $\calU_0$ derived from the scene graph edges $\calE$, automaton transitions, and $\cup_{k > 0} \calU_k$. The initial state and goal region are defined as $x = (s_1,q_1)$ and $\calR = \calV \times \calF$.

The hierarchical planning domain is illustrated in Fig.~\ref{fig:amra_structure}. Four levels, occupancy ($\calV$), objects ($\calA_1$), rooms ($\calA_2$) and floors ($\calA_3$), are used in our experiments. For example, in the object level, the agent can take an action to move directly from the couch to the TV with transition cost computed as the shortest path in the occupancy level.

\subsection{LTL Heuristic} \label{sec:ltl_heuristics}

To ensure optimal AMRA* planning, a consistent heuristic is required at the anchor level. Inspired by the admissible but inconsistent heuristic in \cite{fu2016optimal}, we design a consistent LTL heuristic function $h_{\textsc{LTL}}$ that approximates the scene graph distance to $\calR = \calV \times \calF$ using information from the automaton $\calM_{\phi_{\mu}}$ and the scene graph attributes. 
We require that the scene graph cost $c$ satisfies the triangle inequality, i.e., $c(s, s') \le c(s, s'') + c(s'', s')$ for any $s, s', s'' \in \calV$. Then, we define the cost between two labels $l_1, l_2 \in 2^{\AP}$ as
\begin{equation}
    c_\ell(l_1, l_2) = \min_{s_1, s_2: \ell(s_1)=l_1, \ell(s_2)=l_2} c(s_1, s_2),
\end{equation}
which is a lower bound on the transition cost from $l_1$ to $l_2$ in the metric space of cost function $c$. Next, we define a lower bound on the transition cost from automaton state $q \in \calQ$ with label $l \in 2^{\AP}$ to an accepting state $q_f \in \calF$ as:
\begin{equation}\label{eq:automaton_bellman}
    g(l, q) = \min_{l' \in 2^{\AP}} c_{\ell}(l, l') + g(l', T(q, l')).
\end{equation}
The function $g: 2^{\AP} \times \calQ \mapsto \bbR_{\ge0}$ can be pre-computed via Dijkstra's algorithm on the automaton $\calM_{\phi_{\mu}}$. We also define a next labeling function $\ell_n: 2^{\AP} \times \calQ \mapsto 2^{\AP}$ that tracks the least-cost label sequence returned by Dijkstra's algorithm with $g(l, q) = 0$, $\ell_n(l, q) = \mathsf{true}$, $\forall q \in \calF$.

\begin{proposition}
    The heuristic function $h_{\textsc{LTL}}: \calV \times \calQ \rightarrow \bbR$ defined below is consistent:
    \begin{equation} \label{eq:ltl_heuristics}
    \begin{aligned}
        h_{\textsc{LTL}}(s, q) &= \min_{t \in \calV}\left[ c(s, t) + g(l(t), T(q, l(t))) \right], \\
        h_{\textsc{LTL}}(s, q) &= 0,\quad \forall q \in \calF.
    \end{aligned} 
    \end{equation}
\end{proposition}

\begin{proof}
    Consider a state $x = (s_x, q_x)$ with label $l_x=l(s_x)$. For any $y = (s_y, q_y)$ that $s_y$ is reachable from $s_x$ in one step, we have two cases to handle. When $T(q_x, l(s_y)) = q_y = q_x$:
    \begin{equation}
    \small
    \begin{aligned} \notag
        h_{\textsc{LTL}}&(s_x, q_x) 
            \le \min_{t \in \calV}\left[ c(s_x, s_y) + c(s_y, t) + g(l(t), T(q_x, l(t))) \right] \\
            &= c(s_x, s_y) + \min_{t \in \calV}\left[ c(s_y, t) + g(l(t), T(q_y, l(t))) \right] \\
            &= c(s_x, s_y) + h_{\textsc{LTL}}(s_y, q_y).
    \end{aligned}
    \end{equation}
    When $T(q_x, l_y) = q_y \neq q_x$, with $l_y=l(s_y)$, we have:
    \begin{equation}
    \small
    \begin{aligned} \notag
        &c(s_x, s_y) + h_{\textsc{LTL}}(s_y, q_y) \\
        &= c(s_x, s_y) + \min_{t \in \calV}\left[ c(s_y, t) + g(l(t), T(q_y, l(t))) \right] \\
        &\ge \min_{t\in\calV, l(t)=l_y} c(s_x, t) + \min_{t \in \calV}\left[ c(s_y, t) + g(l(t), T(q_y, l(t))) \right] \\
        &\ge \min_{t\in\calV, l(t)=l_y} c(s_x, t) + \min_{l'\in 2^{\AP}}\left[ c_l(l_y, l') + g(l', T(q_y, l')) \right] \\
        &= \min_{t\in\calV, l(t)=l_y} [c(s_x, t)] + g(l_y, q_y) \\
        &\ge\min_{t\in\calV} \left[c(s_x, t) + g(l(t), T(q_x, l(t)))\right] = h_{\textsc{LTL}}(s_x, q_x).\qedhere
    \end{aligned}
    \end{equation}
\end{proof}

\subsection{LLM Heuristic} \label{sec:llm_heuristics}
In this section, we seek to assist AMRA* by developing an LLM heuristic $h_{\textsc{LLM}}: \calV \times \calQ \rightarrow \bbR$ that captures the hierarchical semantic information of the scene graph. The LLM heuristic uses all attributes at a node $s\in \calV$, the current automaton state $q \in \calQ$, and the attribute hierarchy $\Bar{\calG}$, and returns an attribute-based guide that helps the AMRA$^*$ to search in the right direction for an optimal path.
We design the prompt to ask for LLM attribute-based guidance with $4$ components as follows: 
\begin{itemize}
    \item environment description from its attribute hierarchy $\Bar{\calG}$,
    \item list of motions $M = \{m_i(\cdot, \cdot)\}$, where $m_i(a_j, a_k)$, $a_j\in\calA_j$, $a_k\in\calA_k$ describes movements on $\Bar{\calG}$ from attribute $a_i$ to $a_j$ that the LLM model uses to generate its guides,
    \item an example of the mission $\mu_{unique}$ and how to response,
    \item description of the mission $\mu_{unique}$, current attributes, remaining task given the automaton state $q \in \calQ$, and request for guidance on how to finish the task.
\end{itemize}

The LLM model returns a sequence of function calls $\{f_i(a_j, a_k)\}_{i = 0}^N$, $f_i \in M$, $a_j\in\calA_j$, $a_k\in\calA_k$ in XML format, easing response parsing \cite{vemprala2023chatgpt}. Each function call returns a user-defined cost, e.g., Euclidean distance between attributes: $f_i(a_j, a_k) = c (s_j, s_k)$, where $s_j, s_k$ are the center of $\calV_{a_j}$ and $\calV_{a_k}$, respectively. The total cost of the LLM functions is used as an LLM heuristic $h_{LLM} (s,q) = \sum_{i = 0}^N f_i(a_j, a_k)$. Due to the LLM query delay and its limited query rates, the sequence of function calls suggested by the LLM model is obtained offline stored and used to calculate the heuristic $h_{LLM}$ online in AMRA* based on the user-defined cost.  

\begin{figure}[t]
    \centering
    \includegraphics[width=0.95\linewidth,%
                           trim=0mm 100mm 0mm 0mm,clip]{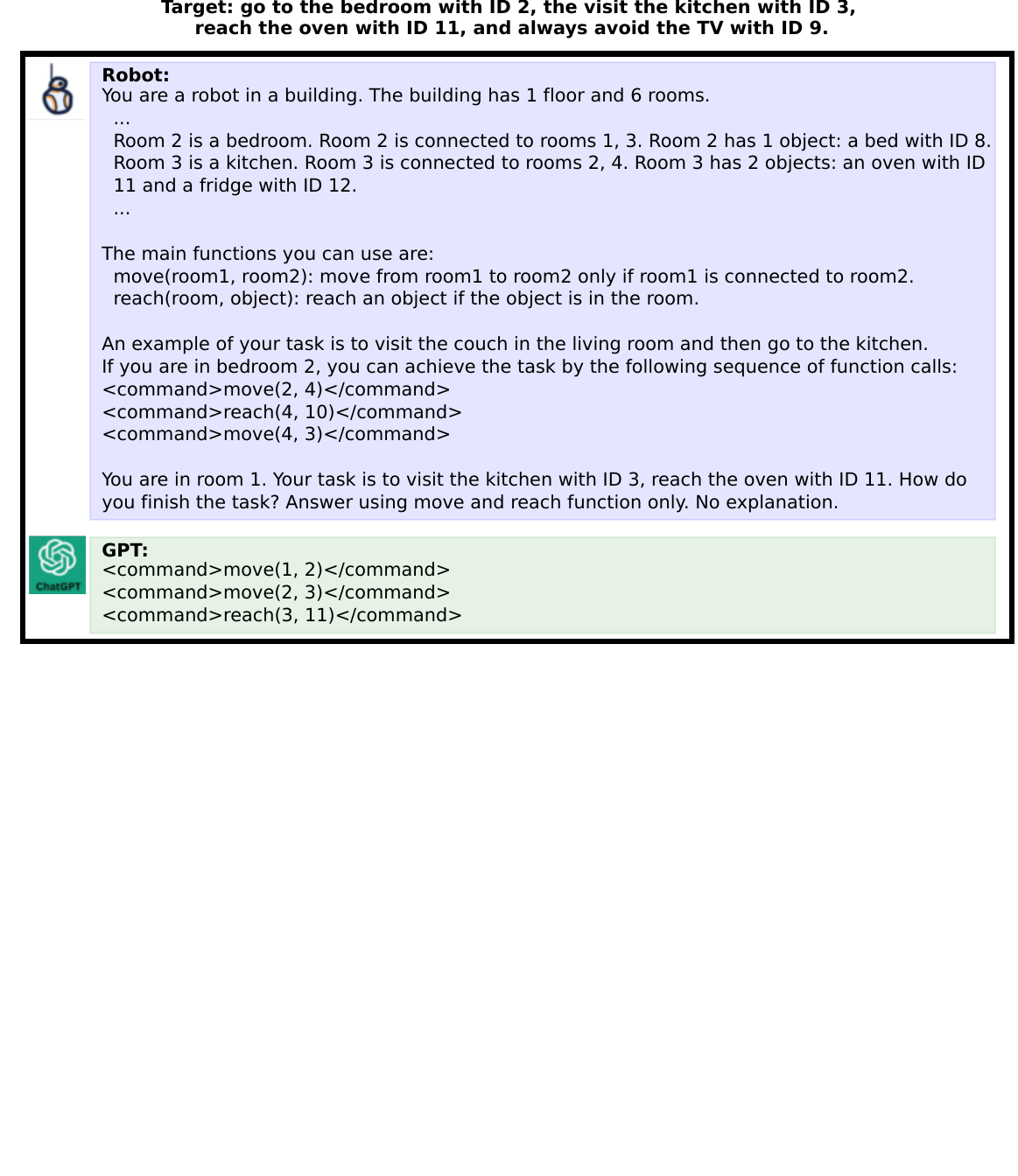}
    \caption{ChatGPT prompt requesting a scene graph path.}
    \label{fig:llm_heuristic_prompt}
\end{figure}

\begin{figure}[t]
    \centering
    \includegraphics[width=0.85\linewidth,,trim=0mm 5mm 0mm 5mm,clip]{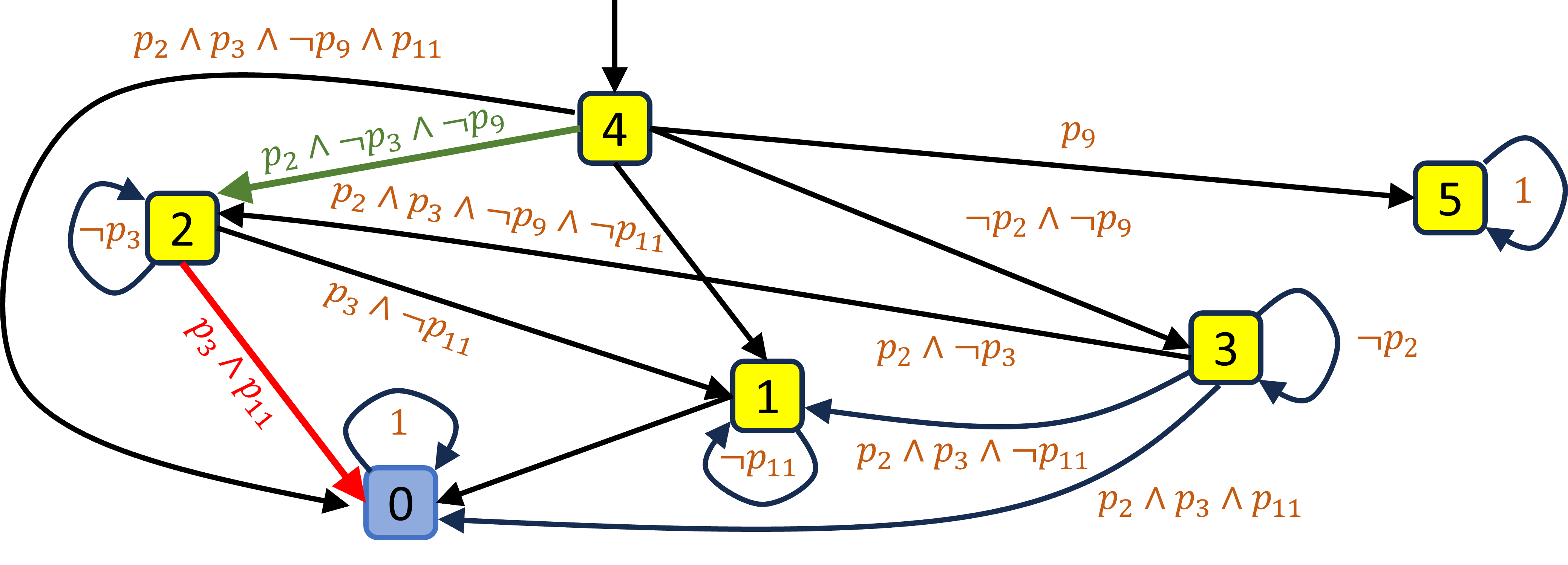}
    \caption{The automaton graph $T$ for the mission \textit{"go to the bedroom 2, then visit the kitchen 3, reach the oven 11, and always avoid the TV 9"} with an initial node $q_1 = 4$ and an accepting node $0$.}
    \label{fig:ltl_automaton_graph_next_task}
\end{figure}

Fig. \ref{fig:llm_heuristic_prompt} illustrates a sample prompt and response. The prompt first describes how the attributes are connected based on the attribute hierarchy $\Bar{\calG}$. Each attribute is mentioned with a unique ID to avoid confusing the LLM, as shown in the first paragraph of the prompt in Fig. \ref{fig:llm_heuristic_prompt}.
The second part of the prompt provides a list of possible functions used to guide the agent, such as $move(1,2)$ to move from room $1$ to room $2$, or $reach(1,3)$ to reach an object $3$ in room $1$. The third component provides an example of a mission and how the LLM should response. The last component describes the current attributes and the remaining mission, generated based on the current automaton state $q$, and requests LLM to generate a high-level plan using the provided functions.

The automaton state $q$ represents how far we have achieved the mission. Thus, to describe the remaining mission, we run Dijkstra's algorithm on the automaton $\calT$ to find the shortest path from $q$ to an accepting state in $\calF$. We obtain a set of atomic prepositions evaluated \textit{true} along the path, and concatenate their descriptions to describe the remaining mission in the prompt. For example, the desired mission is to \emph{``go to the bedroom 2, then visit the kitchen 3, reach the oven 11, while always avoid the TV 9"}. Let the atomic prepositions be defined as $p_2$, $p_3$, $p_{11}$, and $p_9$, where the indices correspond to the ID of the room or object.

The task can be described using an LTL as follows: $\phi = \bfF (p_2 \wedge \bfF (p_3 \wedge \bfF p_{11})) \wedge \neg p_9,$
whose automaton graph $T$ generated from Spot~\cite{spot} is shown in Fig. \ref{fig:ltl_automaton_graph_next_task} with the initial state $q_1 = 4$ and the final states $\calF = \{0\}$. 
The agent is currently in room 1 and have been already visited room 2, i.e. $q = 2$ on $T$. The shortest path from $q$ to the accepting state $0$ is marked by red arrows in Fig. \ref{fig:ltl_automaton_graph_next_task}. Along this path, $p_3$ and $p_{11}$ turn $\mathsf{true}$, causing the remaining mission to be to\emph{``visit the kitchen 3 and reach the oven 11"} (Fig. \ref{fig:llm_heuristic_prompt}). The atomic preposition $p_4$ leads to a sink state $5$ if it evaluates $\mathsf{true}$, and never appears in the next mission, leading to an optimistic LLM heuristic.

%% file: tex/Evaluation.tex
\begin{table}[t]
    \caption{Example mission descriptions for each scene.}
    \label{tab:example_mission_descriptions}
    \centering
    \begin{tabular}{r|l}
        \hline
         \emph{Allensville} & \pbox{6cm}{Clean all vases in the dining room. Eventually water the potted plants in the bathrooms one after another.} \\ \hline
         \emph{Benevolence} & \pbox{6cm}{Visit the bathroom on floor 0 and avoid the sink, then go to the dining room and sit on a chair. Always avoid the living room and the staircase next to it.} \\ \hline
         \emph{Collierville} & \pbox{6cm}{Clean all the corridors, except the one on floor 0.} \\ \hline
    \end{tabular}
\end{table}

\begin{figure}[t]
    \centering
    \includegraphics[width=\linewidth,,trim=0mm 13mm 0mm 13mm,clip]{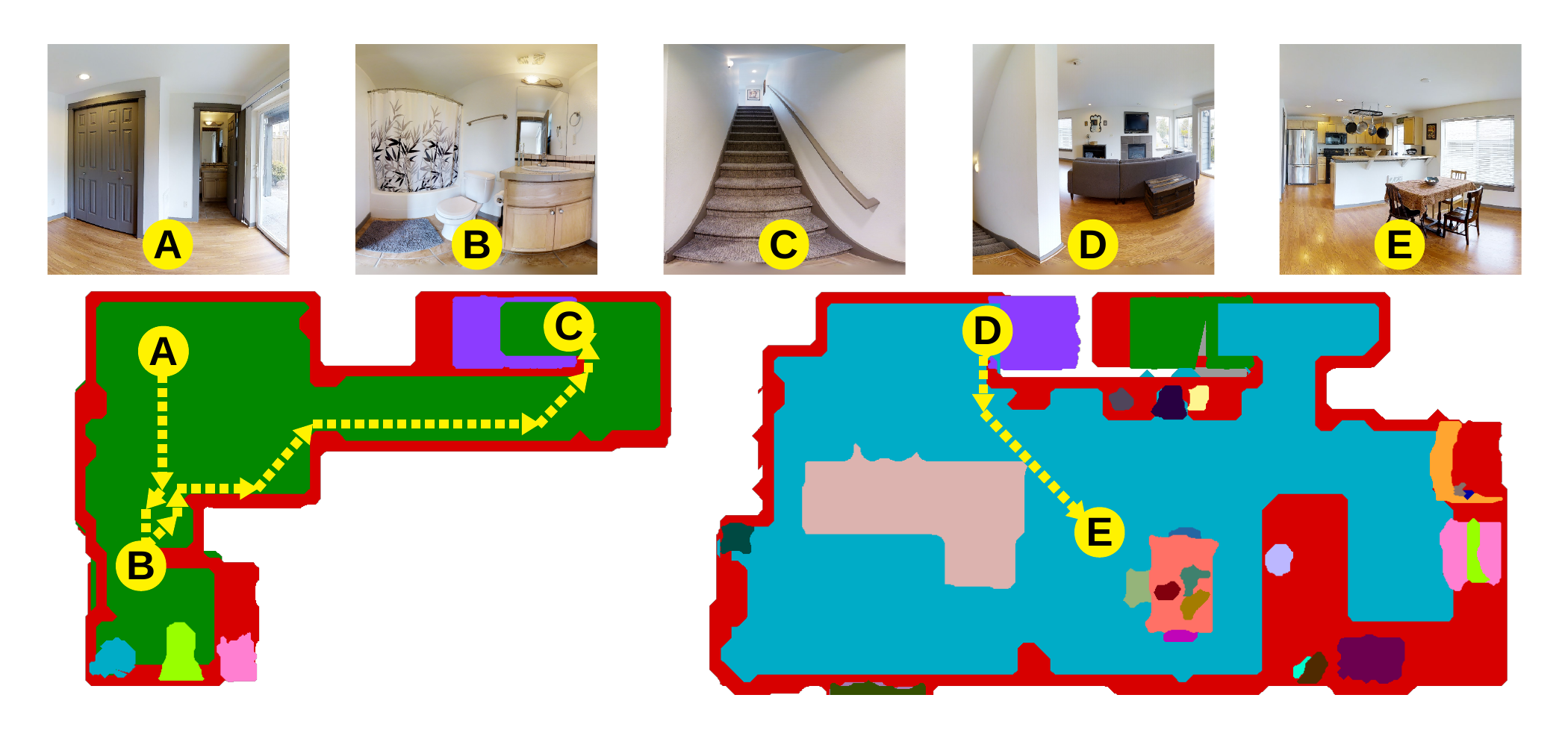}
    \caption{Optimal path for the \emph{Benevolence} mission shown in Table~\ref{tab:example_mission_descriptions}.}
    \label{fig:example_path}
\end{figure}

\begin{figure}[t]
    \centering
    \includegraphics[width=0.48\linewidth,trim=0mm 6mm 0mm 2mm,clip]{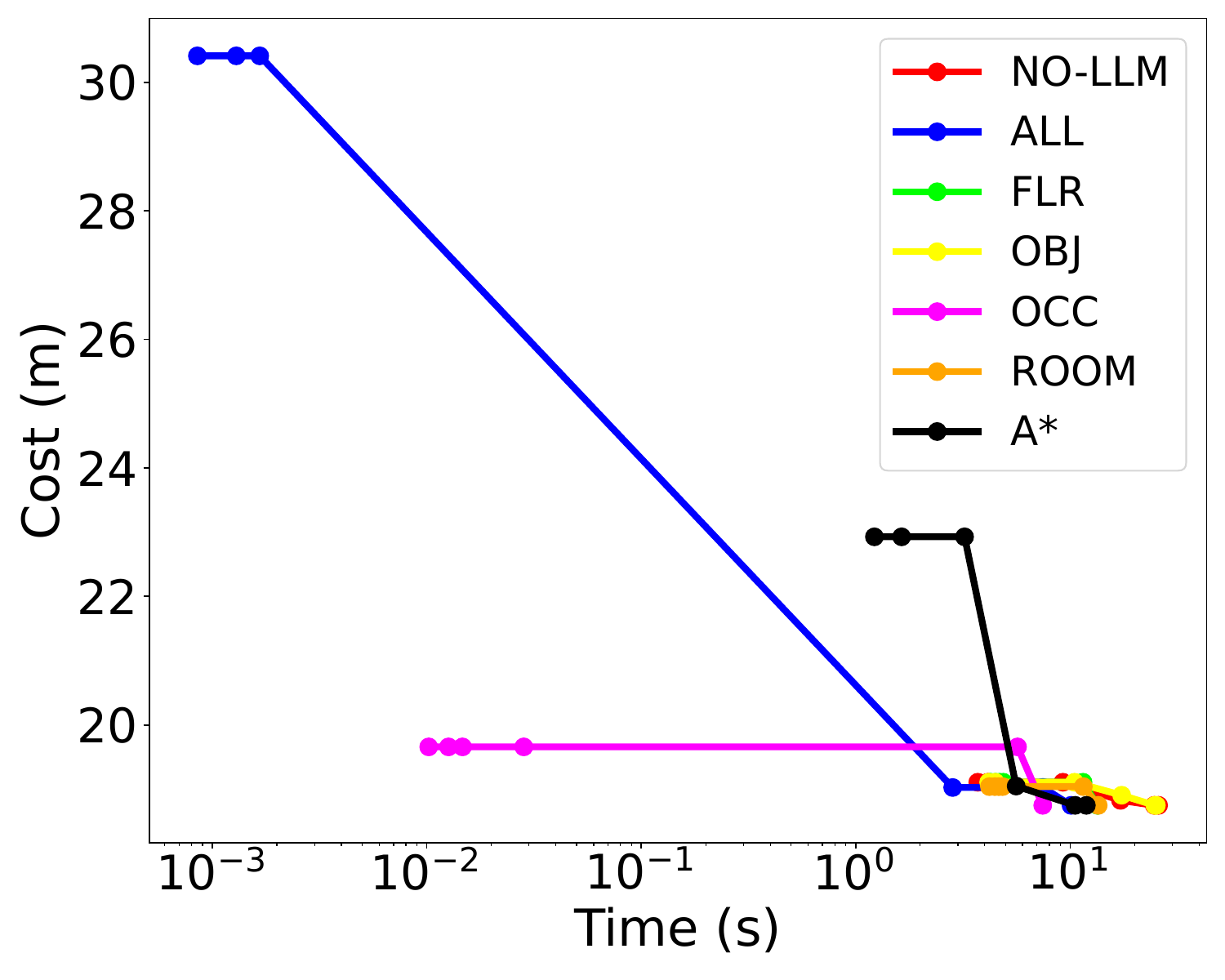}%
    \hfill%
    \includegraphics[width=0.48\linewidth,trim=0mm 6mm 0mm 2mm,clip]{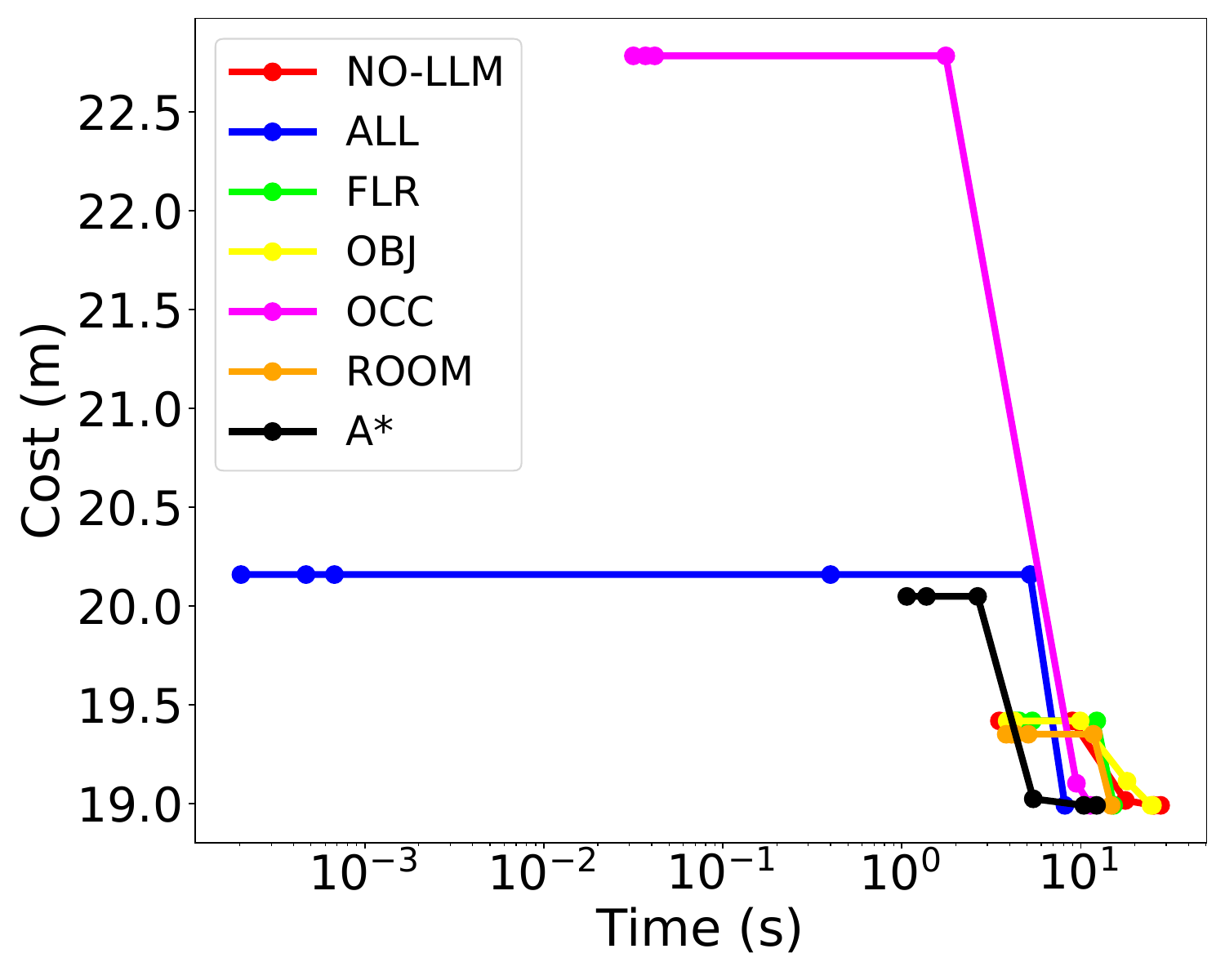}
    \caption{Path cost vs planning time for different AMRA* variants.}
    \label{fig:cost_vs_time}
    \includegraphics[width=0.48\linewidth,trim=0mm 2mm 0mm 2mm,clip]{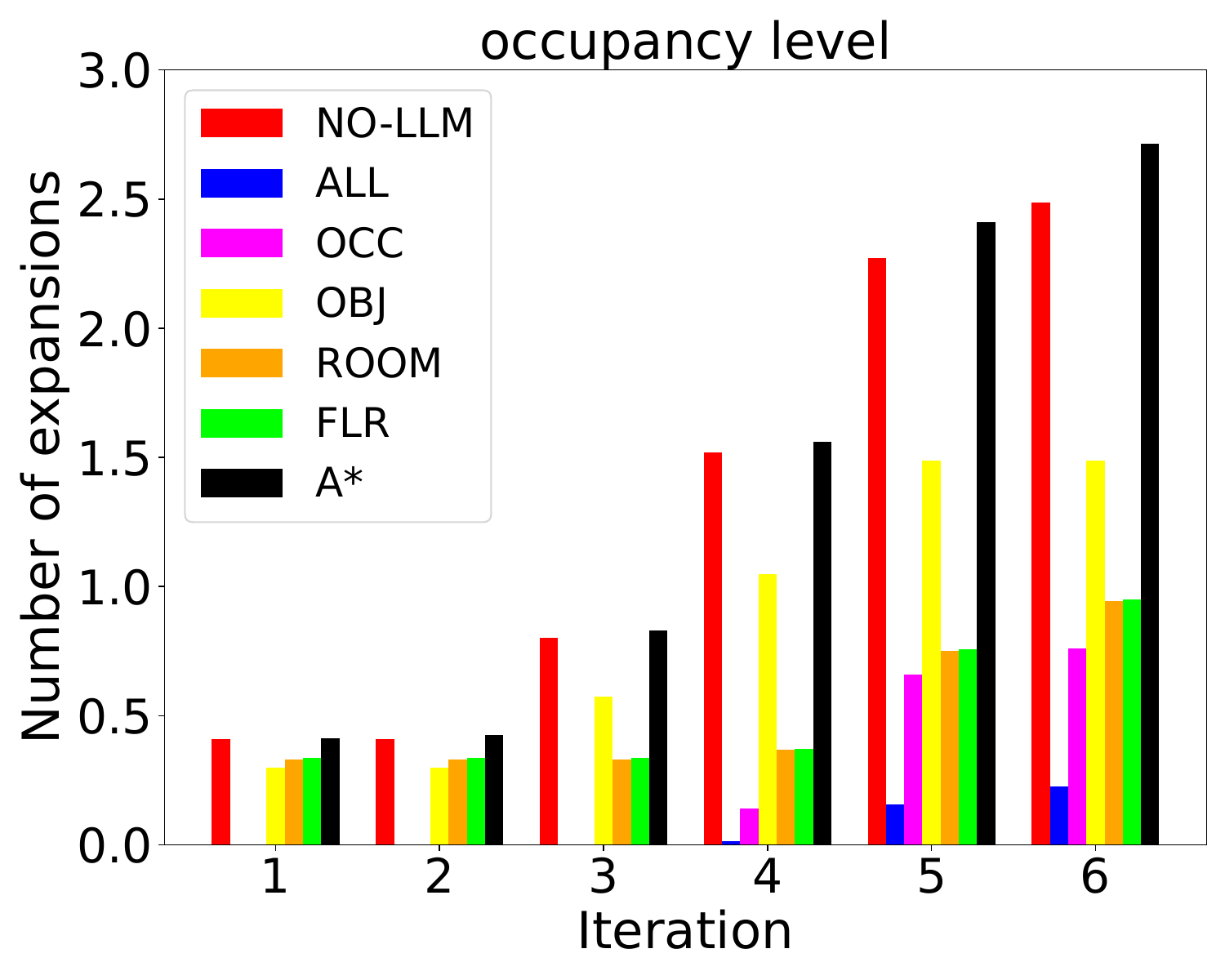}%
    \hfill%
    \includegraphics[width=0.48\linewidth,trim=0mm 2mm 0mm 2mm,clip]{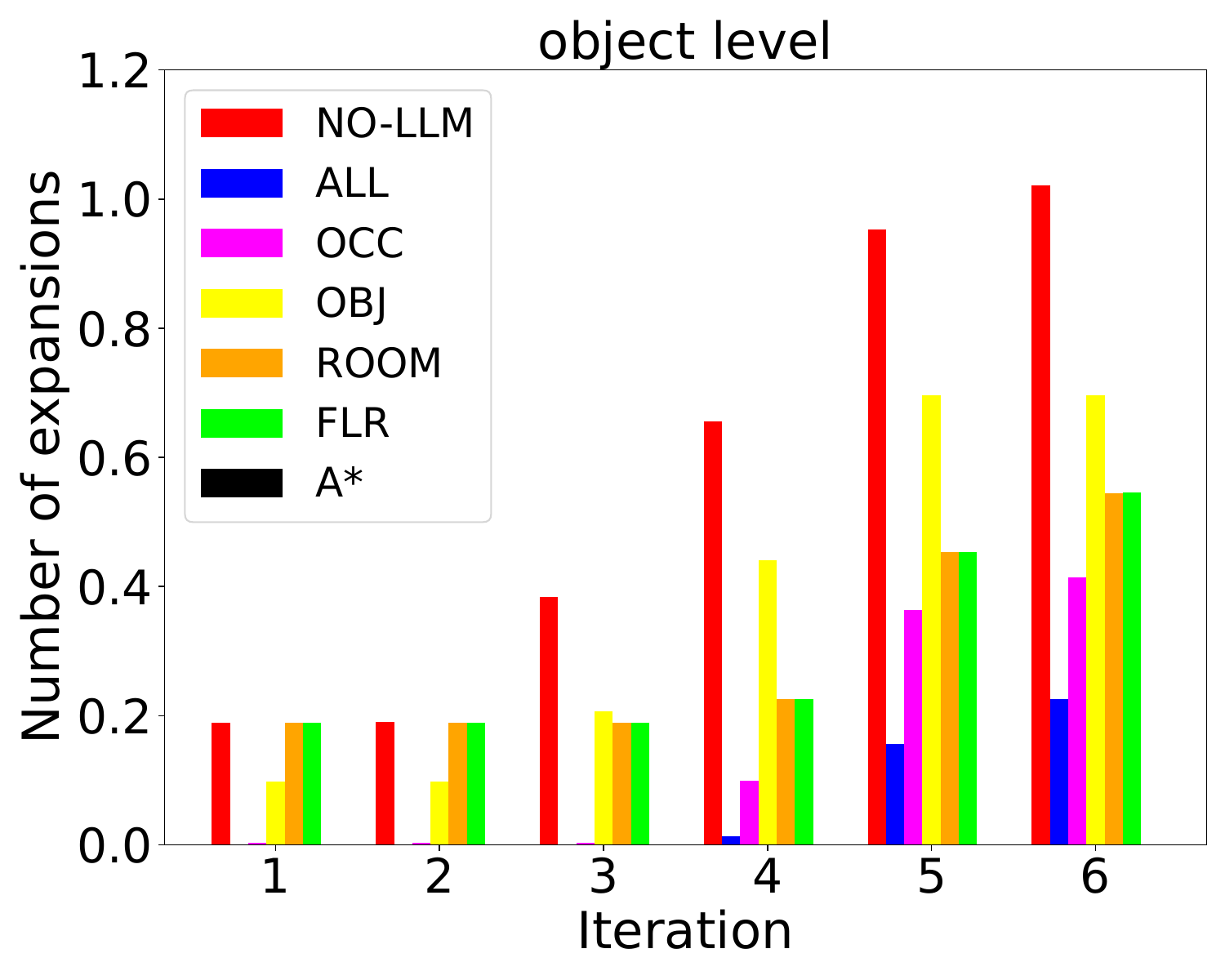}\\
    \includegraphics[width=0.48\linewidth,trim=0mm 6mm 0mm 2mm,clip]{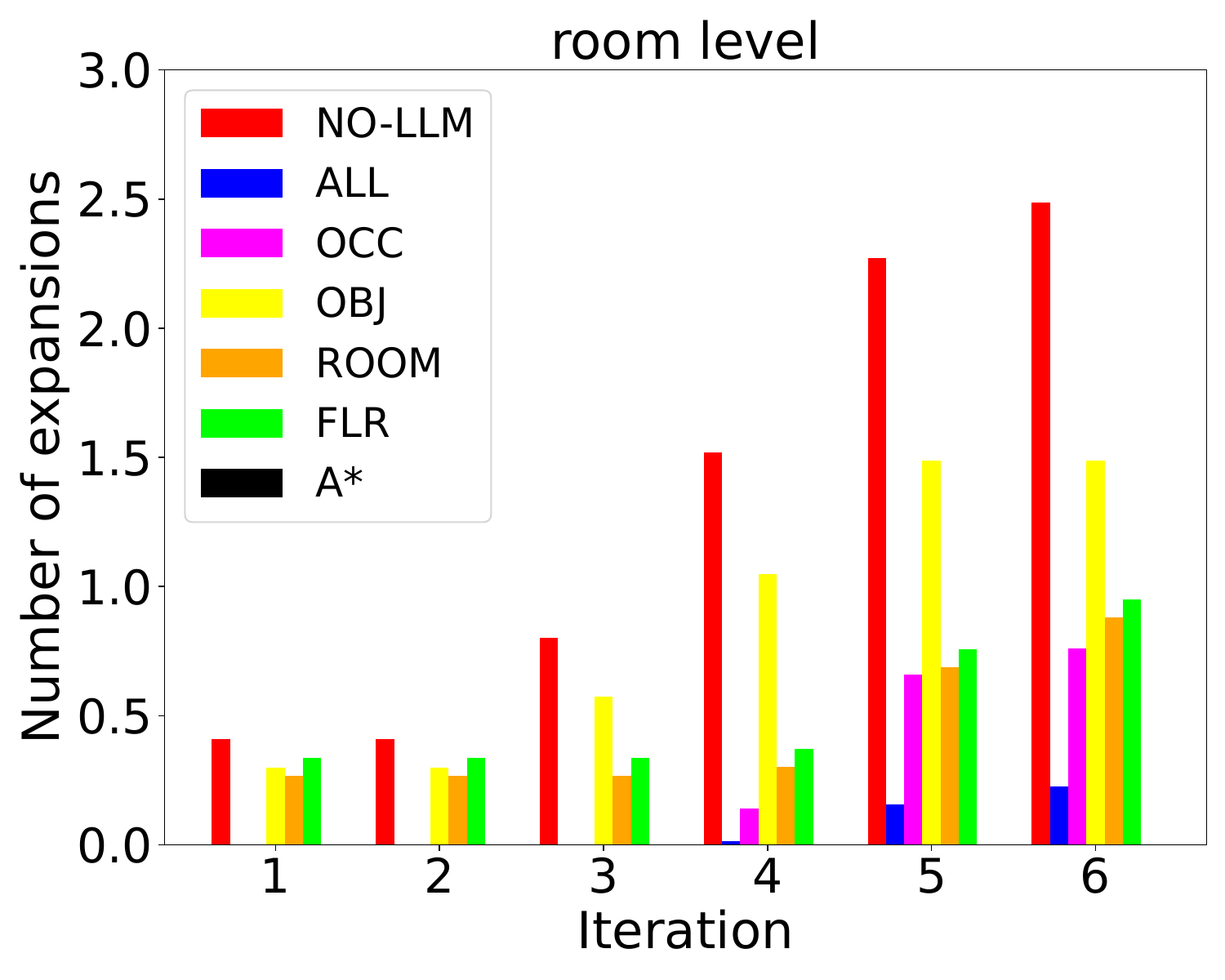}%
    \hfill%
    \includegraphics[width=0.48\linewidth,trim=0mm 6mm 0mm 2mm,clip]{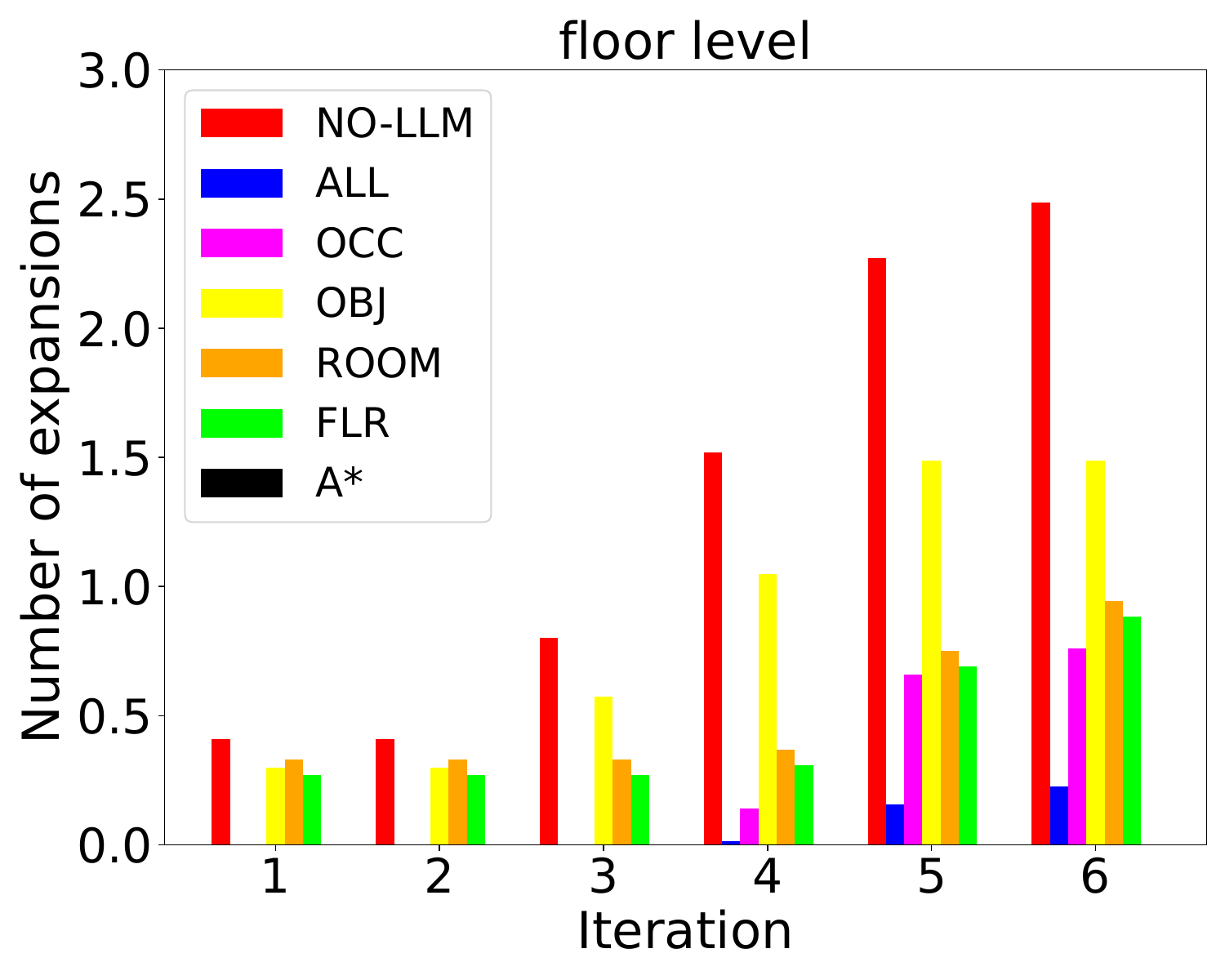}
    \caption{Number of node expansions v.s. the planning iteration. Each sub-figure presents a hierarchy level.}
    \label{fig:expansion_vs_iteration}
\end{figure}

\begin{table}[t]
    \caption{First feasible path computation time, relative cost between first and optimal path, and optimal path computation time averaged over $5$ initial conditions for mission 1 in \emph{Benevolence}.} 
    \label{tab:comp_all_setups}
    \centering
    \resizebox{0.95\linewidth}{!}{
    \begin{tabular}{r|c|c|c}
                 & 1st iter. time (sec.) & 1st iter. $cost / cost_{opt}$ & opt. time (sec.) \\ \hline
        ALL  & \textbf{0.0007} & 1.3763 & 8.9062  \\
        OCC  & 0.0244 & 1.0827 & \textbf{8.3144}  \\
        OBJ  & 3.7460 & 1.0387 & 24.936 \\
        ROOM & 3.6878 & \textbf{1.0306} & 13.1352 \\
        FLR  & 3.8106 & 1.0369 & 13.3287 \\
     NO-LLM  & 3.3260 & 1.0318 & 24.2516 \\
         A*  & 1.1202 & 1.0997 & 11.7594 \\
    \end{tabular}}
    \caption{First path computation time, relative cost between first and optimal path, and optimal path computation time averaged over each scene's 5 missions and 5 initial positions/mission.}
    \label{tab:comp_cross_scenes}
    \centering
    \resizebox{0.98\linewidth}{!}{
    \begin{tabular}{r|c|c|c}
                & \textit{Allensville}(1-floor) & \textit{Benevolence}(3-floor) & \textit{Collierville}(3-floor) \\ \hline
            ALL &  0.24/1.69/3.50 & \textbf{0.32}/1.24/\textbf{11.82} & \textbf{0.009}/1.26/\textbf{5.48} \\
        NO-LLM  &  0.36/1.34/6.33 & 2.82/1.18/21.57 & 1.16/1.38/8.13 \\
            A*  &  \textbf{0.15/1.08/3.23} & 1.33/\textbf{1.16}/14.22 & 1.19/\textbf{1.06}/7.04 \\
    \end{tabular}}
\end{table}

\section{EVALUATION}
\label{sec:evaluation}

To evaluate our method, we use \emph{Allensville} (1-floor), \emph{Benevolence} (3-floor) and \emph{Collierville} (3-floor) from the 3D Scene Graph dataset \cite{armeni20193d}. For each scene, we designed 5 missions (some are shown in Table~\ref{tab:example_mission_descriptions}). For each mission, we used 5 initial positions across different floors and rooms.

We use GPT-4 \cite{openai2023gpt4} to translate missions to LTL formulas, and Spot~\cite{spot} for LTL formulas to automata as described in Sec.~\ref{sec:nl_to_ltl}. Following Sec.~\ref{sec:llm_heuristics}, we use GPT-4 \cite{openai2023gpt4} to generate the LLM heuristic function $h_{\textsc{LLM}}$. Given a scene graph $\calG$, the mission described by the automaton $\calM_\phi$, the LTL heuristic $h_{\textsc{LTL}}$, and the LLM heuristic $h_{\textsc{LLM}}$, we construct the hierarchical planning domain and run AMRA*.

Fig.~\ref{fig:example_path} shows a path in \emph{Benevolence} for the mission in Table \ref{tab:example_mission_descriptions}. Starting from the empty room on floor 0, the agent goes to the bathroom entrance without approaching the sink, and then proceeds upstairs to floor 1, finally reaching a chair in the dining room without entering the living room. 

Since planning domains with different hierarchy levels and different heuristics per level can be constructed, we compare different setups to investigate the benefit of using our LLM heuristic. With all hierarchy levels having an LTL heuristic, we design 7 setups: occupancy level only without LLM heuristic (A*), all levels available but without LLM heuristics (NO-LLM), all levels with LLM heuristics (ALL), and one of the levels with LLM heuristic (OCC, OBJ, ROOM, FLR). Fig.~\ref{fig:cost_vs_time} shows that the ALL setup manages to return a feasible path much faster than others thanks to the LLM guidance across all hierarchical levels, while it also approaches the optimal solution within similar time spans.

As an anytime algorithm, AMRA* starts searching with large weight on the heuristics, then reuses the results with smaller heuristic weights. As the planning iterations increase, the path gets improved. When the heuristic weight decays to 1, we obtain an optimal path. To further investigate the benefits of using LLM heuristics, we compare the number of node expansions per planning iteration. As shown in Fig.~\ref{fig:expansion_vs_iteration}, applying our LLM heuristic to any hierarchy level reduces the node expansions significantly, which indicates that our LLM heuristic produces insightful guidance to accelerate AMRA*. An exciting observation is that the more hierarchy levels we use our LLM heuristic in, the more efficient the algorithm is. This encourages future research to exploit scene semantic information further to accelerate planning.

AMRA* allows a robot to start executing the first feasible path, while the path optimization proceeds in the background. Tables \ref{tab:comp_all_setups} and \ref{tab:comp_cross_scenes} shows the time required to compute the first path, the path cost relative to the optimal path, and the time required to find an optimal path. The ALL configuration outperforms other setups in speed of finding the first path and the optimal path when the scene gets more complicated.

%% file: tex/Conclusion.tex
\section{CONCLUSION}
\label{sec:conclusion}

We demonstrated that an LLM can provide symbolic grounding, LTL translation, and semantic guidance from natural language missions in scene graphs. This information allowed us to construct a hierarchical planning domain and achieve efficient planning with LLM heuristic guidance. We managed to ensure optimality via multi-heuristic planning, including a consistent LTL heuristic. Our experiments show that the LLM guidance is useful at all levels of the hierarchy for accelerating feasible path generation.